\documentclass{article} 
\usepackage{wrapfig}

\usepackage{amsfonts}
\usepackage{amsmath}
\usepackage{amssymb}
\usepackage{amsthm}
\usepackage{appendix}
\usepackage{color}
\usepackage{enumitem}
\usepackage{float}
\usepackage{graphicx}
\usepackage{hyperref}
\usepackage{ltxtable}
\usepackage{mathtools}
\usepackage{mdwlist}
\usepackage{natbib}
\usepackage{nicefrac}
\usepackage{subfigure}
\usepackage{url}
\usepackage{verbatim}
\usepackage{xcolor}
\usepackage{hyperref}

\usepackage{algorithm}    % load everything before hyperref *except* algorithm
\usepackage{algorithmic}

\setlist{topsep=0em,noitemsep}

\makeatletter
\newtheorem*{rep@theorem}{\rep@title}
\newcommand{\newreptheorem}[2]{\newenvironment{rep#1}[1]{\def\rep@title{#2 \ref{##1}}\begin{rep@theorem}}{\end{rep@theorem}}}
\makeatother

\newtheorem{theorem}{Theorem}
\newreptheorem{theorem}{Theorem}

\newreptheorem{lemma}{Lemma}

\newreptheorem{corollary}{Corollary}

\theoremstyle{plain}
\newtheorem{thm}{Theorem}[section]
\newtheorem{lem}[thm]{Lemma}

\theoremstyle{definition}

\theoremstyle{remark}

\theoremstyle{plain}
\newreptheorem{thm}{Theorem}
\newreptheorem{lem}{Lemma}
\newreptheorem{cla}{Claim}
\newreptheorem{prop}{Proposition}
\theoremstyle{definition}
\newreptheorem{defn}{Definition}
\newreptheorem{conj}{Conjecture}
\newreptheorem{exmp}{Example}
\theoremstyle{remark}
\newreptheorem{case}{Case}

\numberwithin{equation}{section}

\newcolumntype{L}{>{\centering\arraybackslash} m{0.04\textwidth}}
\newcolumntype{S}{>{\centering\arraybackslash} m{0.32\textwidth}}

\newcommand{\removedbynati}[1]{}

\newcommand{\expectation}[2][]{\mathbb{E}_{#1}[#2]}

\newcommand{\norm}[1]{\left\lVert {#1} \right\rVert}

\newcommand{\R}{\mathbb{R}}
\newcommand{\N}{\mathbb{N}}

\newcommand{\orthM}{M_\perp}

\newcommand{\spectrum}[2][]{\sigma_{#1}\left( {#2} \right)}
\newcommand{\project}[2][]{\mathcal{P}_{#1}\left(#2\right)}
\DeclareMathOperator{\rank}{rank}
\DeclareMathOperator{\trace}{tr}

\newcommand{\code}[1]{\mbox{\texttt{#1}}}

\title{Stochastic Optimization of PCA with Capped~MSG}

\author{
Raman Arora, Andrew Cotter, and Nathan Srebro\\[10pt]
{Toyota Technological Institute at Chicago}\\
\texttt{arora@ttic.edu, cotter@ttic.edu, nati@ttic.edu} \\
}

\date{}
\begin{document}

\maketitle

\begin{abstract}
We study PCA as a stochastic optimization problem and propose a novel
stochastic approximation algorithm which we refer to as ``Matrix Stochastic
Gradient'' (MSG), as well as a practical variant, Capped MSG. We study the
method both theoretically and empirically.
\end{abstract}

\section{Introduction}\label{sec:introduction}

Principal Component Analysis (PCA) is a ubiquitous tool used in many data
analysis, machine learning and information retrieval applications. It is used
for obtaining a lower dimensional representation of a high dimensional signal
that still captures as much as possible of the original signal. Such a low
dimensional representation can be useful for reducing storage and computational
costs, as complexity control in learning systems, or to aid in visualization.

PCA is typically phrased as a question about a fixed data set: given a data set
of $n$ vectors in $\R^d$, what is the $k$-dimensional subspace that
captures most of the variance in the data set (or equivalently, that is best
in reconstructing the vectors, minimizing the sum squared distances, or
residuals, to the subspace)? It is well known that this subspace is given by
the leading $k$ components of the singular value decomposition of the data
matrix (or equivalently of the empirical second moment matrix). And so, the
study of computational approaches for PCA has mostly focused on methods for
finding the SVD (or leading components of the SVD) for a given $n \times d$
matrix \citep{Oja:85,Sanger:89,streaming}. 

In this paper we approach PCA as a stochastic optimization problem,
where the goal is to optimize a ``population objective'' based on
i.i.d.~draws from the population. That is, in the case of PCA, we
consider a setting in which we have some unknown source
(``population'') distribution $\mathcal{D}$ over $\R^d$, and the goal
is to find the $k$-dimensional subspace maximizing the (uncentered)
variance of $\mathcal{D}$ inside the subspace (or equivalently,
minimizing the average squared residual in the population), based on
i.i.d.~samples from $\mathcal{D}$. The main point here is that the
true objective is not how well the subspace captures the {\em sample}
(i.e.~the ``training error''), but rather how well the subspace
captures the underlying source distribution (i.e.~the ``generalization
error''). Furthermore, we are not concerned here with capturing some
``true'' subspace, and so do not measure the angle to it, but rather
at finding a ``good'' subspace, that is almost as good as the optimal
one.

Of course, finding the subspace that best captures the sample is a very reasonable approach to PCA on the population. This is essentially an Empirical Risk Minimization (ERM) approach.  However,
when comparing it to alternative, perhaps computationally cheaper,
approaches, we argue that one should not compare the error on the
sample, but rather the population objective. Such a view can justify
and favor computational approaches that are far from optimal on the
sample, but are essentially as good as ERM {\em on the population}.

Such a population-based view of optimization has recently been advocated in
machine learning, and has been used to argue for crude stochastic approximation
approaches (online-type methods) over sophisticated deterministic optimization
of the empirical (training) objective (i.e.~``batch'' methods)~
\citep{BottouBo07,Shalev:08}. A similar argument was also made in the context
of stochastic optimization, where \cite{Nemirovski:09} argues for stochastic
approximation (SA) approaches over ERM. 
Accordingly, SA approaches, mostly
variants of Stochastic Gradient Descent, are often the methods of choice for
many learning problems, especially when very large data sets are
available~\citep{Shalev:07,Collins:08,Shalev:09}. We would like to take the
same view in order to advocate for, study, and develop stochastic approximation
approaches for PCA.

In an empirical study of stochastic approximation methods for PCA, a
heuristic ``incremental'' method showed very good empirical
performance \citep{Allerton}.  However, no theoretical guarantees or
justification were given for incremental PCA.  In fact, it was shown
that for some distributions it can converge to a suboptimal solution
with high probability (see Section \ref{sec:incremental} for more about this
``incremental'' algorithm).  Also relevant is careful theoretical work
on online PCA by \citet{Warmuth:08}, in which an online regret
guarantee was established.  Using an online-to-batch conversion, this
online algorithm can be converted to a stochastic approximation algorithm with
good iteration complexity, however the runtime for each iteration is
essentially the same as that of ERM (i.e.~of PCA on the sample), and
thus senseless as a stochastic approximation method (see Section~\ref{sec:meg} for more on this algorithm).

In this paper we borrow from these two approaches and present a novel
algorithm for stochastic PCA---the Matrix Stochastic Gradient (MSG)
algorithm.  MSG enjoys similar iteration complexity to Warmuth's and
Kuzmin's algorithm, and in fact we present a unified view of both
algorithms as different instantiations of Mirror Descent for the same
convex relaxation of PCA.  We then present the capped MSG, which is a
more practical variant of MSG, has very similar updates to those of
the ``incremental'' method, and works well in practice, and does not get
stuck like the ``incremental'' method.  The Capped MSG is thus a
clean, theoretically well founded method, with interesting connections
to other stochastic/online PCA methods, and excellent practical
performance---a ``best of both worlds'' algorithm.

\section{Problem Setup}

We consider PCA as the problem of finding the maximal (uncentered)
variance $k$-dimensional subspace with respect to an (unknown) {\em
  distribution} $\mathcal{D}$ over $x \in \R^d$. We assume without
loss of generality a scaling such that $\expectation[ x \sim
\mathcal{D}]{ \norm{x}^2 } \le 1$.  We also require for our
analysis a bounded fourth moment: $\expectation[ x \sim \mathcal{D}]{
  \norm{x}^4 } \le 1$.
We represent a $k$-dimensional subspace by an orthonormal basis,
collected in the columns of a matrix $U$.  With this parametrization,
PCA is defined as the following stochastic optimization problem, 
\begin{align}
\label{eq:objective} \mbox{maximize}: &\ \  \expectation[x\sim\mathcal{D}]{x^T UU^T x} \\
\notag \mbox{subject to} : &\ \ U \in \R^{d\times k}, U^TU=I.
\end{align}
In a stochastic optimization setting we do not have direct knowledge
of the distribution and have access to it only through
i.i.d.~samples---these can be thought of as ``training examples''. As
with other studies of stochastic approximation methods, we are less
concerned with the number of required samples, but rather with the
overall runtime required to obtain an $\epsilon$-suboptimal solution.

The standard approach to \eqref{eq:objective} is empirical risk
minimization (ERM): given samples $\{x_t\}_{t=1}^T$, from the
distribution, we compute the empirical covariance matrix
${\hat{C}=\frac{1}{T} \sum_{t=1}^T x_t x_t^T}$, 
and pick the columns of $U$ to be the eigenvectors of $\hat{C}$ corresponding to
the top-$k$ eigenvalues.  This approach requires $O(d^2)$ memory and
$O(d^2)$ operations just in order to compute the covariance matrix,
plus some additional time for the SVD.  We are interested in methods
with much lower sample time and space complexity, preferably linear
rather than quadratic in $d$.

\section{MSG and MEG}
\label{sec:sgd}

A natural stochastic approximation (SA) approach to PCA is to perform
projected stochastic gradient descent (SGD) on Problem
\ref{eq:objective}, with respect to the variable $U$.  This leads to
the \emph{stochastic power method} with each iteration given as
\begin{equation}
U^{(t+1)} = \project[orth]{ U^{(t)} + \eta x_t x_t^T }, \nonumber 
\end{equation}
where, $x_t x_t^T$ is the gradient of the PCA objective w.r.t. $U$,
$\eta$ is a step size, and $\project[orth]{\cdot}$ projects its
argument onto the set of orthogonal matrices.  Unfortunately, although
SGD is well understood for convex problems, Problem \ref{eq:objective}
is non-convex.  Consequently, obtaining a theoretical understanding of the
stochastic power method, or of how the step size should be set, has
proved elusive.  Under some conditions, convergence to the optimal
solution can be ensured, but no rate is known \citep{Oja:85,Sanger:89,Allerton}.

Instead, we consider a re-parameterization of the PCA problem where
the objective is convex. Instead of representing a linear subspace in
terms of its basis matrix, $U$, we parametrize it using the
corresponding projection matrix $M=UU^T$. We can now reformulate the PCA problem as 
\begin{align}
\label{eq:objective2} \mbox{maximize}: &\ \  \expectation[x\sim\mathcal{D}]{x^T M x} \\
\notag \mbox{subject to} : &\ \ M \in \R^{d\times d}, \spectrum[i]{M} \in \left\{0, 1 \right\}, \rank M = k,
\end{align}
where $\spectrum[i]{M}$ is the $i^{th}$ eigenvalue of $M$. 

We now have a convex (even linear) objective, but the constraints in
\eqref{eq:objective2} are not convex.  This prompts us to consider its
convex relaxation:
\begin{align}
\label{eq:convex-objective} \mbox{maximize} : &\ \
\expectation[x\sim\mathcal{D}]{x^T M x} \\
\notag \mbox{subject to} : &\ \ M \in \R^{d\times d}, 0 \preceq M \preceq I,
\trace M = k. \quad \quad \ \ \ \ \ 
\end{align}
Since the objective is linear, and the constraint set of
\eqref{eq:convex-objective} is just the convex hull of the constraints
of \eqref{eq:objective2}, an optimum of \eqref{eq:convex-objective} is
always attained at a ``vertex'', i.e.~a point on the boundary of the
original constraints \eqref{eq:objective2}.  The optimum of
\eqref{eq:objective2} and \eqref{eq:convex-objective} are thus the
same (strictly speaking---every optimum of \eqref{eq:objective2} is
also an optimum of \eqref{eq:convex-objective}), and solving
\eqref{eq:convex-objective} is equivalent to solving
\eqref{eq:objective2}.

Furthermore, even if some $\epsilon$-suboptimal solution we find for
\eqref{eq:convex-objective} is not rank-$k$, i.e. is not a feasible
point of \eqref{eq:objective2}, we can easily sample from it a
rank-$k$ solution, feasible for \eqref{eq:objective2}, with the same
value (in expectation).  This follows from the following result of \cite{Warmuth:08}. 
\begin{lem}[Rounding \citep{Warmuth:08}]
\label{lem:rounding}
Any feasible solution of \eqref{eq:convex-objective} can be expressed as a  convex combination of at most $d$ feasible solutions of \eqref{eq:objective2}. 
\end{lem}
Furthermore, Algorithm 4.1 of \cite{Warmuth:08} shows how to
efficiently find such a convex combination. Since the objective is
linear, treating the coefficients of the convex combination as
sampling weights and choosing randomly among the $d$ components yields
a rank-$k$ matrix with the desired objective function value, in
expectation.

\subsection{Matrix Stochastic Gradient}
\label{sec:msg}

Performing SGD on the convex Problem \ref{eq:convex-objective}
(w.r.t.~the variable $M$) yields the following iterates:
\begin{equation}
\label{eq:sgd-update} M^{(t+1)} = \project{ M^{(t)} + \eta x_t x_t^T }, 
\end{equation}
where the projection is now performed onto the (convex) constraints
of \eqref{eq:convex-objective}.  The {\bf Matrix Stochastic Gradient (MSG)}
algorithm entails:
\begin{enumerate}
\item Choose step-size $\eta$, iteration count $T$, and starting point $M^{(0)}$.
\item Iterate the updates \eqref{eq:sgd-update} $T$ times, each time
  using an independent sample $x_t \sim \mathcal{D}$.
\item Average the iterates as $\bar{M} = \frac{1}{T} \sum_{t=1}^{T}
  M^{(t)}$. 
\item Sample a rank-$k$ solution $\tilde{M}$ from $\bar{M}$ using the
  rounding procedure discussed in the previous section.
\end{enumerate}
Analyzing MSG is straightforward using the standard SGD analysis
\citep{Nemirovski:83}:
\begin{theorem}
\label{lem:sgd-rate}
After $T$ iterations of MSG (on Problem \ref{eq:convex-objective}),
with step size $\eta = \sqrt{ \frac{k}{T} }$, and starting at
$M^{(0)}=0$,
\begin{equation*}
\expectation[]{ \expectation[x\sim\mathcal{D}]{ x^T \tilde{M} x}} \geq \expectation[x\sim\mathcal{D}]{x^T M^* x} - \frac{1}{2} \sqrt{\frac{k}{T}, }
\end{equation*}
where the expectation is w.r.t.~the i.i.d.~samples
$x_1,\ldots,x_T\sim\mathcal{D}$ and the rounding, and $M^*$ is the
optimum of \eqref{eq:objective2}. 
\end{theorem}
\begin{proof}
Standard SGD analysis of \citet{Nemirovski:83} yields that 
\begin{equation}
\label{eq:sgdanal}
\expectation[]{x^T M^* x - x^T \bar{M} x} \leq \frac{\eta}{2} \expectation[x\sim\mathcal{D}]{\| g \|^2_F } + \frac{\|M^* - M^{(0)}\|^2_F}{2 \eta T},
\end{equation}
where $g=xx^T$ is the gradient of the PCA objective. Now,
$\expectation[x\sim\mathcal{D}]{\| g \|^2_F }=\expectation[x \sim
\mathcal{D}]{\|x\|^4} \leq 1$ and $\norm{ M^* - M^{(0)} }_F^2 =
\norm{M^*}_F^2 = k$.  In the last inequality, we used the fact that
$M^*$ has $k$ eigenvalues of value $1$ each, and hence $\norm{M^*}_F
= \sqrt{k}$.
\end{proof}

\subsection{Efficient Implementation and Projection}
\label{sec:effproj}

\begin{algorithm}[t]
\begin{small}
\begin{tabbing}
\textbf{mm}\=mm\=mm\=mm\=mm\=mm\=mm\=mm\=mm\=\kill
\>$\code{msg-step}\left( d,k,m:\N, U':\R^{d \times m}, \sigma':\R^{m}, x:\R^d, \eta:\R \right)$\\
\>\textbf{1}\'\>$\hat{x} \leftarrow \sqrt{\eta}(U')^T x$; $x_\perp \leftarrow \sqrt{\eta}x - U' \hat{x}$; $r \leftarrow \norm{x_\perp}$;\\
\>\textbf{2}\'\>$\code{if }r > 0$\\
\>\textbf{3}\'\>\>$V, \sigma \leftarrow \code{eig}( [ \code{diag}( \sigma' ) + \hat{x} \hat{x}^T, r \hat{x} ; r \hat{x}^T, r^2 ] )$;\\
\>\textbf{4}\'\>\>$U \leftarrow [ U', x_\perp / r ] V$;\\
\>\textbf{5}\'\>$\code{else}$\\
\>\textbf{6}\'\>\>$V, \sigma \leftarrow \code{eig}( \code{diag}( \sigma' ) + \hat{x} \hat{x}^T )$;\\
\>\textbf{7}\'\>\>$U \leftarrow U' V$;\\
\>\textbf{8}\'\>$\sigma \leftarrow$ distinct eigenvalues in $\sigma$; $\kappa \leftarrow $ corresponding multiplicities;\\
\>\textbf{9}\'\>$\sigma \leftarrow$ \code{project} $(d,k,m,\sigma,\kappa)$;\\
\>\textbf{10}\'\>$\code{return } U, \sigma$;
\end{tabbing}
\end{small}

\caption{
\small
Matrix stochastic gradient (MSG) update: compute an eigendecomposition of $M' + \eta x x^T$ from a rank-$n$ eigendecomposition $M' = U' \code{diag}(\sigma') (U')^T$ and project the resulting solution onto the constraint set. The computational cost of this algorithm is dominated by the matrix multiplication defining $U$
(line $4$ or $7$) costing $O(m^2 d)$ operations.
}
\label{alg:rank1-update}
\end{algorithm}

A naive implementation of the MSG update requires $O(d^2)$ memory and
$O(d^2)$ operations per iteration.  In this section, we show how to
perform this update efficiently by maintaining an up-to-date
eigendecomposition of
$M^{(t)}$. Pseudo-code for the update is given as Algorithm \ref{alg:rank1-update}.
Consider the eigendecomposition $M^{(t)} =U' \code{diag}(\sigma)
(U')^T$, at the $t^{th}$ iteration, where $\rank(M^{(t)})=k_t$ and $U' \in \R^{d \times k_t}$. Given a new observation $x_t$, the eigendecomposition of $M^{(t)} +
\eta x_t x_t^T$ can be updated efficiently using a $(k_t+1) \times
(k_t + 1)$ SVD \citep{Brand:02,Allerton} (steps 1-7 of Algorithm~\ref{alg:rank1-update}).  This rank-one eigen-update
is followed by projection onto the constraints of
\eqref{eq:convex-objective}, invoked as \texttt{project} in step 8 of
Algorithm~\ref{alg:rank1-update}, discussed in the following paragraphs
and given as Algorithm~\ref{alg:project}. The projection procedure is based on the following lemma\footnote{Note that our projection problem onto the capped simplex, even when seen in the vector setting, is substantially different from \cite{Duchi:2008}. We project onto the set $\{0 \leq \sigma \leq 1, \|\sigma\|_1=k\}$ in \eqref{eq:convex-objective} and $\{0 \leq \sigma \leq 1, \|\sigma\|_1=k, \|\sigma\|_0 \leq K\}$ in \eqref{eq:capped-convex-objective} whereas \cite{Duchi:2008} project onto $\{0 \leq \sigma, \|\sigma\|_1=k\}$. 
}:
\begin{lem}\label{lem:sgd-projection}
Let $M'\in\R^{d\times d}$ be a symmetric matrix, with eigenvalues 
$\sigma_{1}',\dots,\sigma_{d}'$ and associated eigenvectors $v_{1}',\dots,v_{d}'$. 
Its projection $M=\project{M'}$ onto the feasible region
of Problem \ref{eq:convex-objective} with respect to the Frobenius
norm, is the unique feasible matrix which has the same eigenvectors
as $M'$, with the associated eigenvalues $\sigma_{1},\dots,\sigma_{d}$
satisfying:
\begin{equation*}
\sigma_{i} = \max\left(0,\min\left(1,\sigma_{i}'+S\right)\right)
\end{equation*}
with $S\in\R$ being chosen in such a way that
$\sum_{i=1}^{d}\sigma_{i}=k$.
\end{lem}
\begin{proof}
In Appendix \ref{app:proofs}.  
\end{proof}

This result shows that projecting onto the feasible region amounts to finding
the value of $S$ such that, after shifting the eigenvalues by $S$ and clipping
the results to $[0,1]$, the result is feasible.
Importantly, the projection operates \emph{only} on the eigenvalues.  Algorithm \ref{alg:project} contains
pseudocode which finds $S$ from a list of eigenvalues. It is optimized to
efficiently handle repeated eigenvalues---rather than receiving the eigenvalues
in a length-$d$ list, it instead receives a length-$n$ list containing only the
\emph{distinct} eigenvalues, with $\kappa$ containing the corresponding
multiplicities. In Sections \ref{subsec:low-rank} and \ref{subsec:capped-rank},
we will see why this is an important optimization.

The central idea motivating the algorithm is that, in a sorted array of
eigenvalues, all elements with indices below some threshold $i$ will be clipped
to $0$, and all of those with indices above another threshold $j$ will be
clipped to $1$. The pseudocode simply searches over all possible pairs of such
thresholds until it finds the one that works.

The rank-one eigen-update combined with the fast projection step yields
an efficient MSG update that requires $O(d k_t)$ memory and $O(d
k_t^2)$ operations per iteration, where recall that $k_t$ is the rank
of the iterate $M^{(t)}$.  This is a significant improvement
over the $O(d^2)$ memory and $O(d^2)$ computation required by a standard
implementation of MSG, if the iterates have relatively low rank.

\begin{algorithm}[t]

\begin{small}
\begin{tabbing}
\textbf{mm}\=mm\=mm\=mm\=mm\=mm\=mm\=mm\=mm\=\kill
\>$\code{project}\left( d,k,n:\N, \sigma':\R^{n}, \kappa':\N^{n} \right)$\\
\>\textbf{1}\'\>$\sigma',\kappa' \leftarrow \code{sort}( \sigma',\kappa' )$;\\
\>\textbf{2}\'\>$i \leftarrow 1$; $j \leftarrow 1$; $s_i \leftarrow 0$; $s_j \leftarrow 0$; $c_i \leftarrow 0$; $c_j \leftarrow 0$;\\
\>\textbf{3}\'\>$\code{while } i \le n$\\
\>\textbf{4}\'\>\>$\code{if } ( i < j )$\\
\>\textbf{5}\'\>\>\>$S \leftarrow ( k - ( s_j - s_i ) - ( d - c_j ) ) / ( c_j - c_i )$;\\
\>\textbf{6}\'\>\>\>$b \leftarrow ($\\
\>\textbf{7}\'\>\>\>\>$( \sigma_i' + S \ge 0 ) \code{ and } ( \sigma_{j-1}' + S \le 1 )$\\
\>\textbf{8}\'\>\>\>\>$\code{and } ( ( i \le 1 ) \code{ or } ( \sigma_{i-1}' + S \le 0 ) )$\\
\>\textbf{9}\'\>\>\>\>$\code{and } ( ( j \ge n ) \code{ or } ( \sigma_{j+1}' \ge 1 ) )$\\
\>\textbf{10}\'\>\>\>$)$;\\
\>\textbf{11}\'\>\>\>$\code{return }S \code{ if } b$;\\
\>\textbf{12}\'\>\>$\code{if } ( j \le n ) \code{ and } ( \sigma_j' - \sigma_i' \le 1 )$\\
\>\textbf{13}\'\>\>\>$s_j \leftarrow s_j + \kappa_j' \sigma_j'$; $c_j \leftarrow c_j + \kappa_j'$; $j \leftarrow j + 1$;\\
\>\textbf{14}\'\>\>$\code{else}$\\
\>\textbf{15}\'\>\>\>$s_i \leftarrow s_i + \kappa_i' \sigma_i'$; $c_i \leftarrow c_i + \kappa_i'$; $i \leftarrow i + 1$;\\
\>\textbf{16}\'\>$\code{return error}$;
\end{tabbing}
\end{small}

\caption{
\small
Routine which finds the $S$ of Lemma \ref{lem:sgd-projection}. It takes as
parameters the dimension $d$, ``target'' subspace dimension $k$, and the number
of \emph{distinct} eigenvalues $n$ of the current iterate. The length-$n$
arrays $\sigma'$ and $\kappa'$ contain the distinct eigenvalues and their
multiplicities, respectively, of $M'$ (with $\sum_{i=1}^{n} \kappa_i' = d$).
Line $1$ sorts $\sigma'$ and re-orders $\kappa'$ so as to match this sorting.
The loop will be run at most $2n$ times (once for each possible increment to
$i$ or $j$ on lines $12$--$15$), so the computational cost is dominated by that
of the sort: $O(n \log n)$.
}

\label{alg:project}

\end{algorithm}

\subsection{Matrix Exponentiated Gradient}
\label{sec:meg}
Since $M$ is constrained by its trace, and not by its Frobenius norm,
it is tempting to consider mirror descent (MD) \citep{BeckTe03} instead 
of SGD updates for solving Problem \ref{eq:convex-objective}.  Recall
that the Mirror Descent updates depend on a choice of ``potential
function'' $\Psi(\cdot)$ which should be chosen according to the 
geometry of the feasible set {\em and} the subgradients \citep{UniMD}.
Using the squared Frobenius norm as a potential function, i.e.
$\Psi(M)=\|M\|^2_F$, yields SGD, i.e.~the MSG updates \eqref{eq:sgd-update}.
The trace-norm constraint suggests using the von Neumann entropy of
the spectrum as the potential function, i.e. $\Psi_h(M) = \sum_i
\lambda_i(M) \log \lambda_i(M)$ where $\lambda_i$ are the eigenvalues
of $M$.  This leads to multiplicative updates which we refer to as
Matrix Exponentiated Gradient (MEG) update similar to those presented
by \citep{Warmuth:08}.  In fact, Warmuth and Kuzmin's algorithm
exactly corresponds to online Mirror Descent on
\eqref{eq:convex-objective} with this potential function, but taking
the optimization variable to be $\orthM = I-M$ (with the constraints
$\trace \orthM = d-k$ and $0 \preceq \orthM \preceq I$).  In either
case, using the entropy potential, despite being well suited for the
trace-geometry, does not actually lead to better
dependence\footnote{This is because in our case, due to the other
  constraints, $\norm{M^*}_F=\sqrt{\trace M^*}$.  Furthermore, the SGD
  analysis depends on the Frobenius norm of the stochastic gradients,
  but since all stochastic gradients are rank one, this is the same as
  their spectral norm, which comes up in the entropy-case analysis,
  and again there is no benefit.} on $d$ or $k$, and Mirror Descent
analysis again yields an excess loss of $\sqrt{k/T}$.  Warmuth and
Kuzmin do present an ``optimistic'' analysis, with a dependence on the
``reconstruction error'' $L^* = \expectation{x^T(I-M^*)x}$, which
yields an excess error of $O\left( \sqrt{\frac{L^* k \log(d/k)}{T}} +
  \frac{k\log(d/k)}{T} \right)$ (their logarithmic term can be avoided
by a more careful analysis).

\section{MSG runtime and the rank of the iterates}
\label{subsec:low-rank}

As we saw, MSG requires $O(k/\epsilon^2)$ iterations to obtain an
$\epsilon$-suboptimal solution and each iteration of MSG costs
$O(k_t^2 d)$ operations where $k_t$ is the rank of iterate $M^{(t)}$.
This yields a total runtime of $O(\bar{k^2} dk /\epsilon^2)$, where
$\bar{k^2}=\sum_{t=1}^T k^2_t$. Clearly, the runtime for MSG depends
critically on the rank of the iterates. If the rank of the
iterates is as large as $d$, MSG achieves a runtime that is cubic in
the dimensionality. On the other hand, if the rank of the iterates is
$O(k)$, the runtime is linear in the dimensionality.
Fortunately, in practice the ranks are typically much lower than the
dimensionality. The reason for this is that MSG performs a rank-$1$
update followed by a projection onto the constraints.  Since $M'
=M^{(t)} + \eta x_t x_t^T$ will have a \emph{larger} trace than
$M^{(t)}$ (i.e. $\trace M' \ge k$), the projection, as is shown by
Lemma \ref{lem:sgd-projection}, will \emph{subtract} a quantity $S$
from every eigenvalue of $M'$, clipping each to $0$ if it becomes
negative. Therefore, each MSG update will increase the rank of the
iterate by at most $1$, and has the potential to decrease it, perhaps
significantly. It's very difficult to theoretically quantify how the
rank of the iterates will evolve over time, but we have observed
empirically that the iterates do tend to have relatively low rank. 

We explore this issue in greater detail experimentally, on a distribution which we expect to be difficult for MSG. To this end, we generated data from known $32$-dimensional distributions with diagonal covariance matrices $\Sigma=\code{diag}(\sigma/\norm{\sigma})$, where 
$\sigma_i =\tau^{-i}/\sum_{j=1}^{32} \tau^{-j}$, for $i=1,\ldots,32$ and for some $\tau > 1$. Observe that $\Sigma^{(k)}$ has a smoothly-decaying set of eigenvalues and the rate of decay is controlled by $\tau$. As $\tau \to 1$, the spectrum becomes flatter resulting in distributions that present challenging test cases for MSG. We experimented with $\tau=1.1$ and $k\in\{1,2,4\}$, where $k$ is the desired subspace dimension used by each algorithm. The data is generated by sampling the $i^{th}$ standard unit basis vector $e_i$ with probability $\sqrt{\Sigma_{ii}}$. We refer to this as the ``orthogonal distribution'', since it is a discrete distribution over $32$ orthogonal vectors. 

\begin{figure}
\begin{center}
\begin{tabular}{cc}
{$k_t'$} & {Spectrum}\\  
\includegraphics[width=0.48\textwidth]{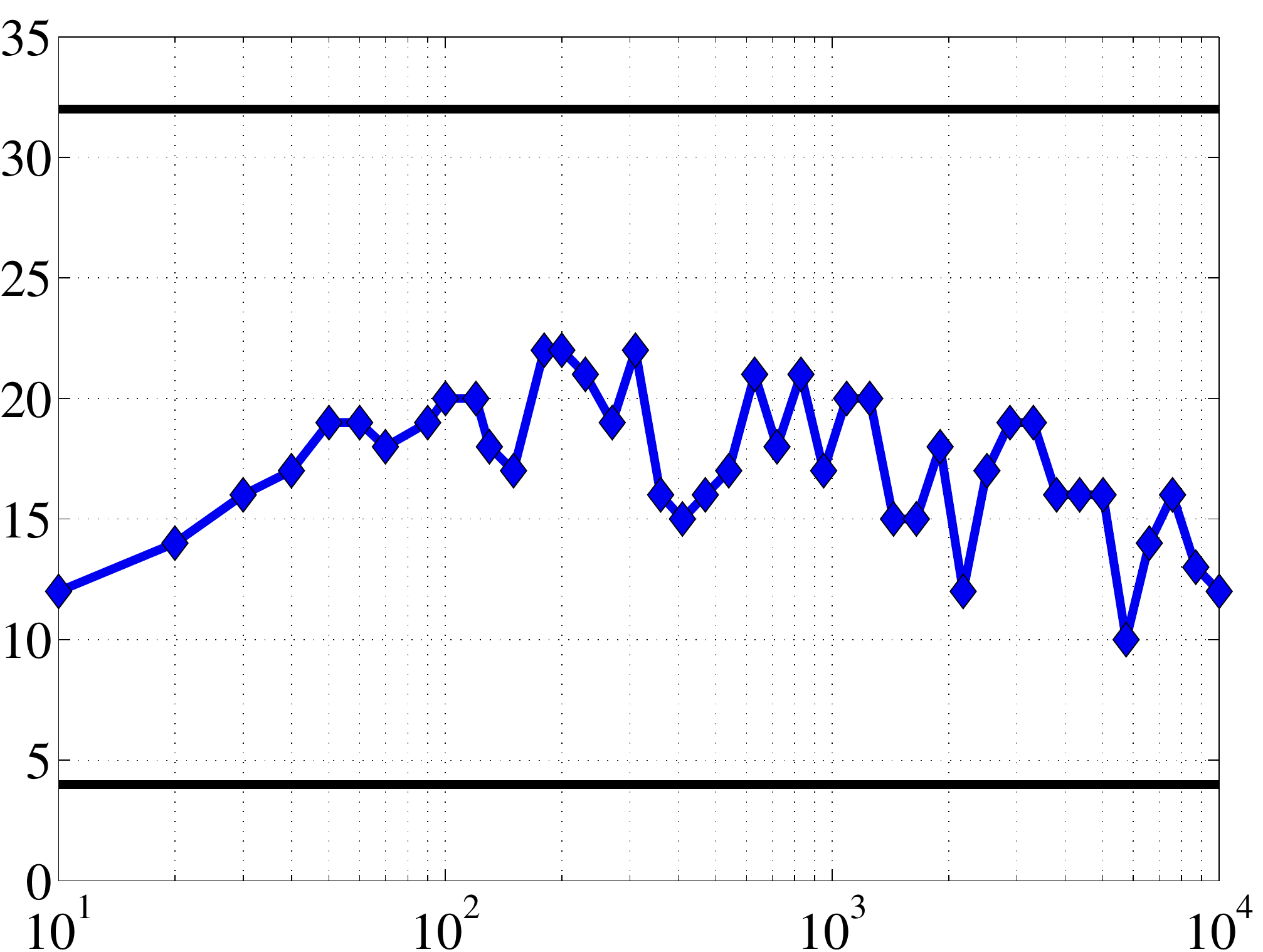} & 
\includegraphics[width=0.48\textwidth]{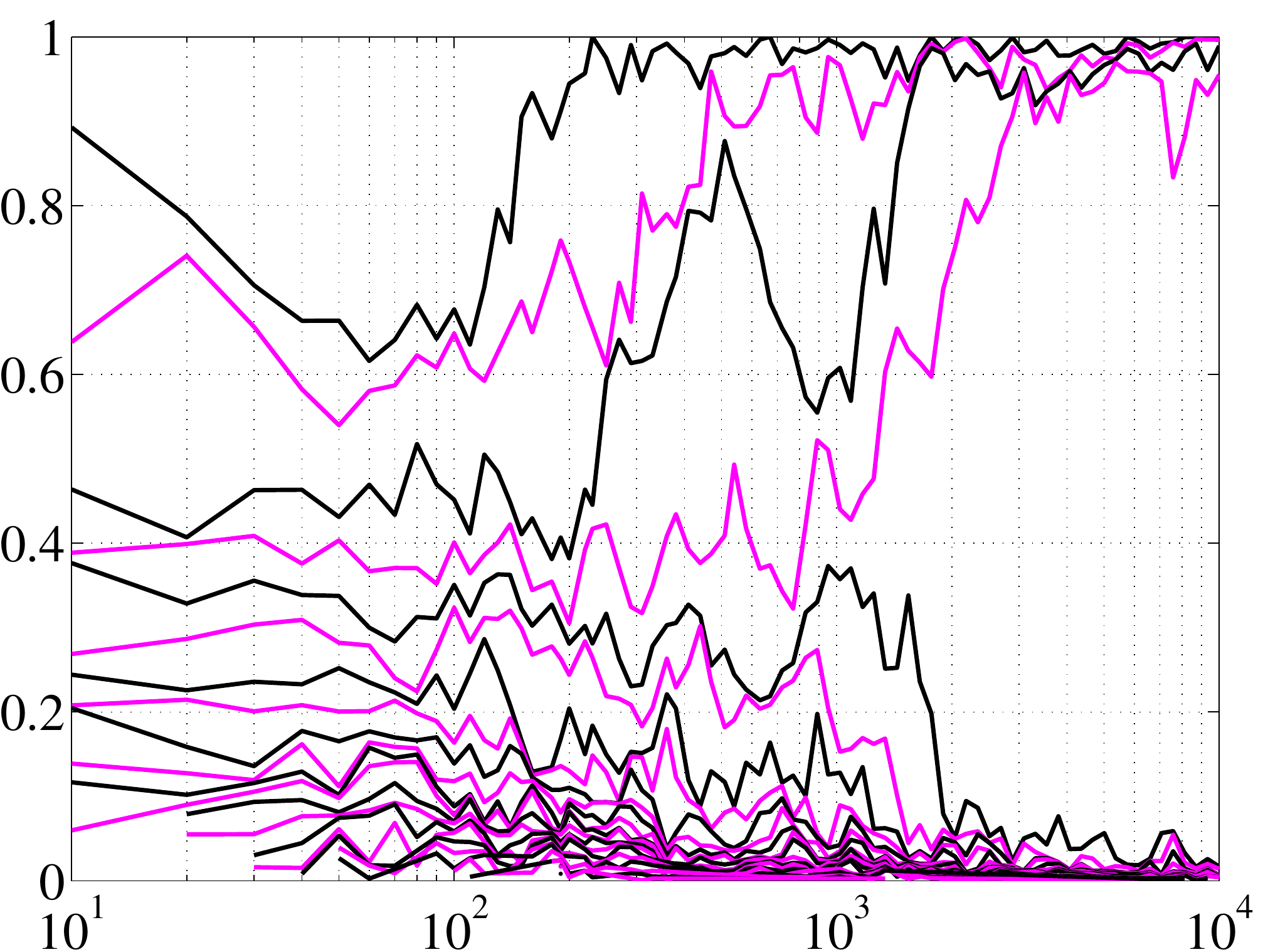} 
\\
\scriptsize{Iterations} & \scriptsize{Iterations} 
\end{tabular}
\end{center}
\caption{
\small
The ranks $k_t'$ (left) and the eigenvalues (right) of the MSG iterates $M^{(t)}$. 
}
\label{fig:simulated-runtime}
\end{figure}

In Figure \ref{fig:simulated-runtime}, we show the results 
with $k=4$. We can see from the left-hand plot that MSG algorithm maintains a subspace of dimension around $15$. The plot on the right shows how the set of nonzero eigenvalues of the MSG iterates evolves over time, from which we can see that many of the extra dimensions are ``wasted'' on very small eigenvalues, corresponding to directions which leave the state matrix only a handful of iterations after they enter. This suggests that constraining $k_t'$ can lead to significant speedups and  motivates capped MSG updates discussed in the next section.

\section{Capped MSG}\label{subsec:capped-rank}
While, as was observed in the previous section, MSG's iterates will tend to
have ranks $k_t'$ smaller than $d$, they will nevertheless also be larger than
$k$. For this reason, in practice, we recommend adding a hard constraint $K$ on
the rank of the iterates:
\begin{align}
\label{eq:capped-convex-objective} \mbox{maximize} : & \ \
\expectation[x\sim\mathcal{D}]{x^T M x} \\
\notag \mbox{subject to} : &\ \ M \in \R^{d\times d}, 0 \preceq M \preceq I \\
\notag  &\ \ \trace M = k, \rank M \le K
\end{align}
We will refer MSG where the projection is replaced with a projection
onto the constraints of \eqref{eq:capped-convex-objective} (i.e.~where
the iterates are SGD iterates on \eqref{eq:capped-convex-objective})
as ``capped MSG''.  For similar reasons as discussed before, as long
as $K\geq k$, Problem \ref{eq:capped-convex-objective} and Problem
\ref{eq:convex-objective} have the same optimum, and it is achieved at
a rank-$k$ matrix, and the extra rank constraint in \ref{eq:capped-convex-objective} is inactive at the optimum. However, the rank 
constraint does affect the iterates, especially since Problem
\ref{eq:capped-convex-objective} is no longer convex.  Nonetheless if
$K>k$ (i.e.~the hard rank-constraint $K$ is {\em strictly} larger than
the target rank $k$), we can easily check if we are at a global
optimum of \ref{eq:capped-convex-objective}, and hence of
\ref{eq:convex-objective}: if the capped MSG algorithm converges to a
solution of rank $K$, then the upper bound $K$ should be increased.
Conversely, if it has converged to a rank-deficient solution, then it
must be the global optimum.  There is thus an advantage in using
$K>k$, and we recommend setting $K=k+1$, as we do in our experiments,
and increasing $K$ only if a rank deficient solution is not found.

Setting $K=k$, the only way to satisfy the trace constraint is to have
all non-zero eigenvalues be equal to one, and
\eqref{eq:capped-convex-objective} becomes identical to
\eqref{eq:objective2}.  The detour through the convex problem
\eqref{eq:convex-objective}, allows us to increase the search rank
$K$, allowing for more flexibility in the search, while still
encouraging the desired rank $k$ through the rank constraint.

\subsection{Implementing the projection}
Implementing capped MSG is similar to implementing MSG (Algorithm~\ref{alg:rank1-update}) except for the projection step. Reasoning as in the proof of Lemma \ref{lem:sgd-projection} shows that if $M^{(t+1)}\!=\!\mathcal{P}\left( M' \right)$ with $M' = M^{(t)} + \eta x_t x_t^T$, then $M^{(t)}$ and $M'$ are simultaneously diagonalizable, and therefore we can consider only how the projection acts on the eigenvalues. Hence, if we let $\sigma'$ be the vector of the eigenvalues of $M'$, and suppose that there are more than $K$ such eigenvalues, then there is a size-$K$ subset of $\sigma'$ such~that applying Algorithm \ref{alg:project} to this set gives the projected eigenvalues. Since we perform only a rank-$1$ update at every iteration, we must check at most $K$ possibilities, at a total cost of~$O(K^2 \log K)$ operations, with no effect on asymptotic runtime because Algorithm \ref{alg:rank1-update} requires $O(K^2 d)$ operations.

\subsection{Relationship to the incremental PCA method}
\label{sec:incremental}
The capped MSG updates with $K=k$ are similar to the
incremental algorithm of \citet{Allerton}. The incremental algorithm
maintains a rank-$k$ approximation of the covariance matrix with
updates given by
\begin{equation*}
{M}^{(t+1)} = \project[\textrm{rank-$k$}]{{M}^{(t)} + x_t x_t^T}, 
\end{equation*}
where the projection is onto the set of rank-$k$ matrices.  Unlike
MSG, incremental updates do not have a step-size.  Updates can be
performed efficiently much in the same way as described in
Section~\ref{sec:effproj}, by maintaining the eigendecomposition of
the iterates.

\begin{figure*}[t]
\begin{center}
\begin{tabular}{ @{} L @{} S @{} S @{} S @{} }
& \large{$k=1$} & \large{$k=2$} & \large{$k=4$} \\
\rotatebox{90}{\scriptsize{Suboptimality}} &
\includegraphics[width=0.32\textwidth]{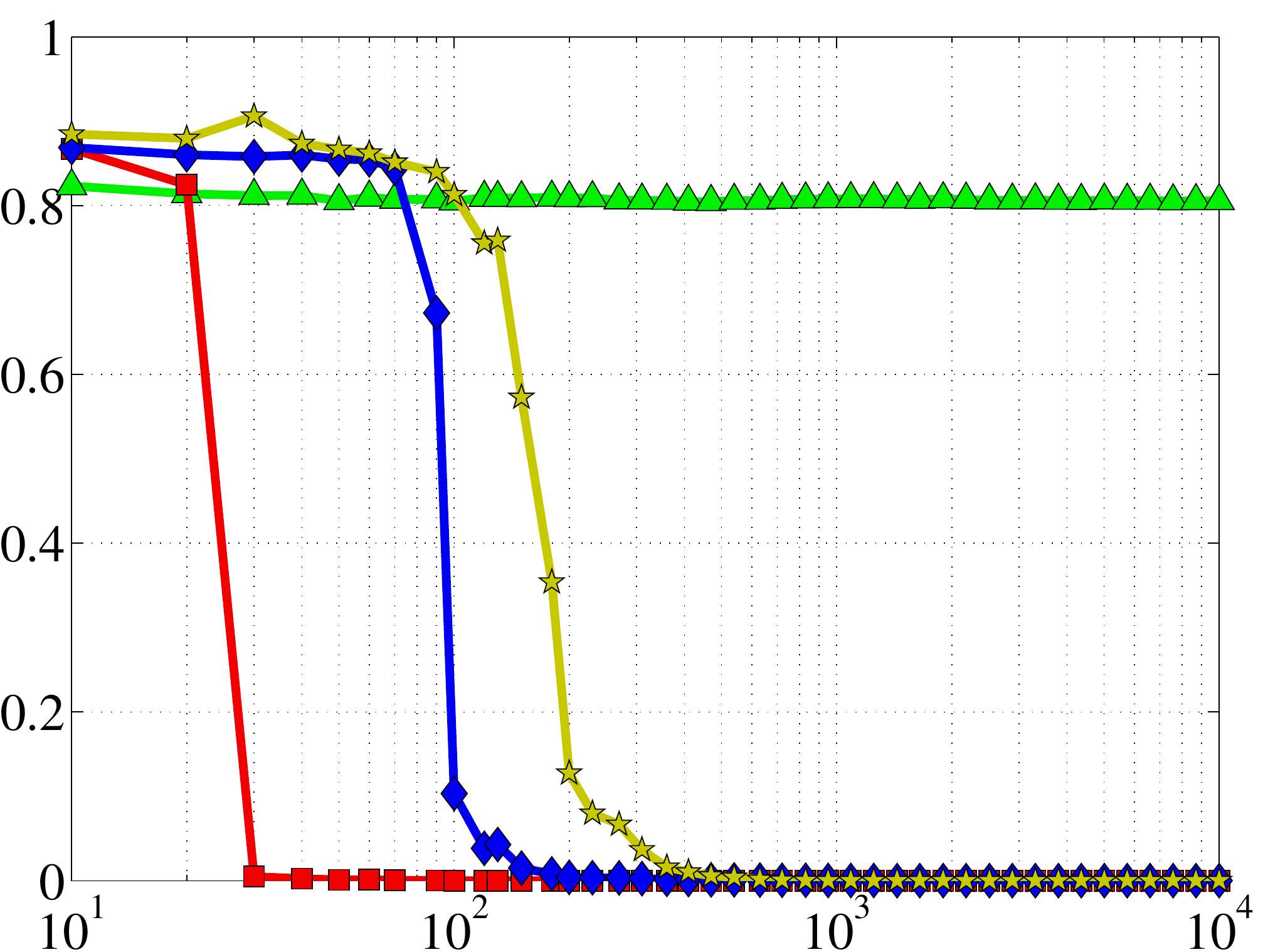} &
\includegraphics[width=0.32\textwidth]{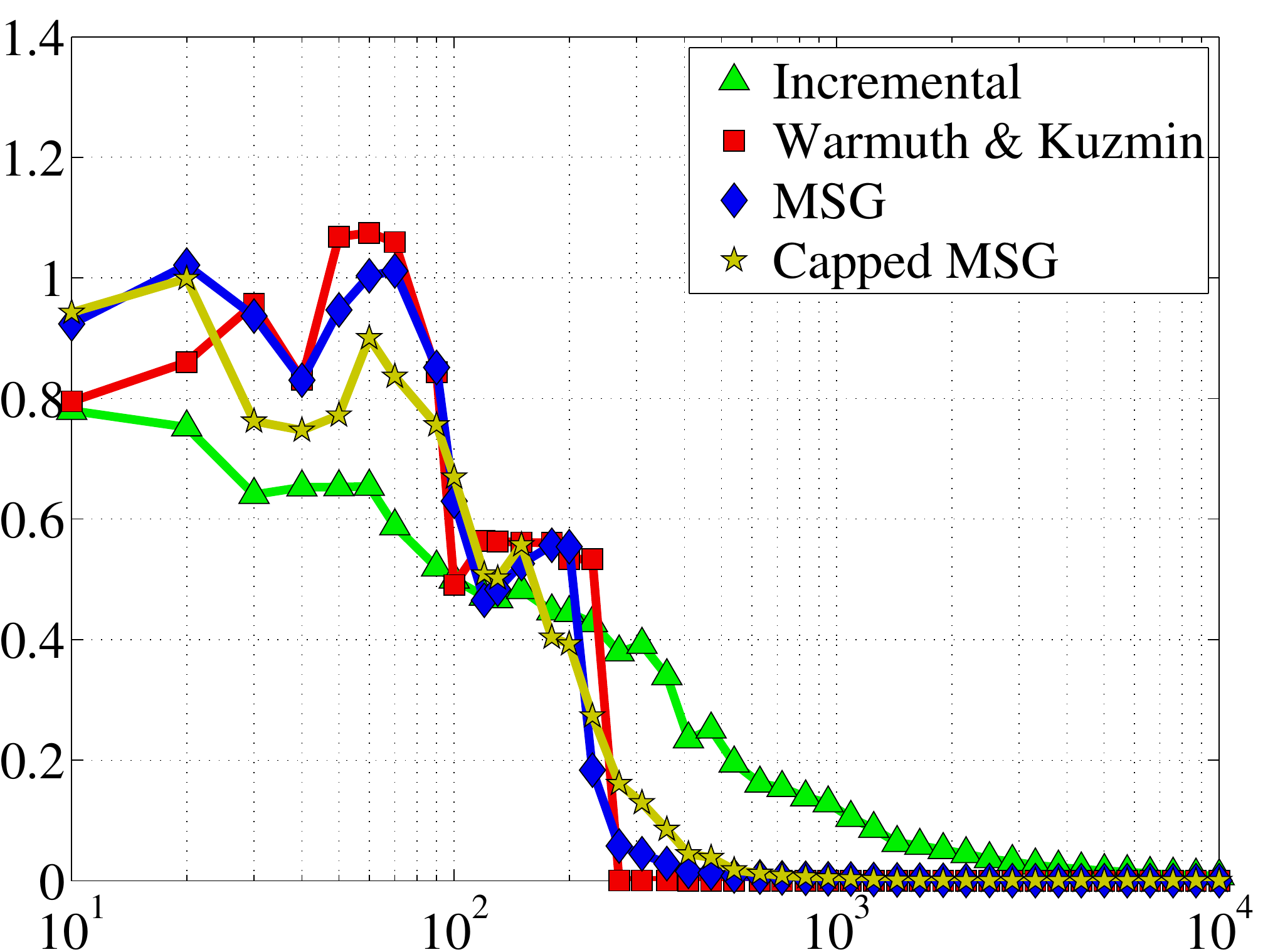} &
\includegraphics[width=0.32\textwidth]{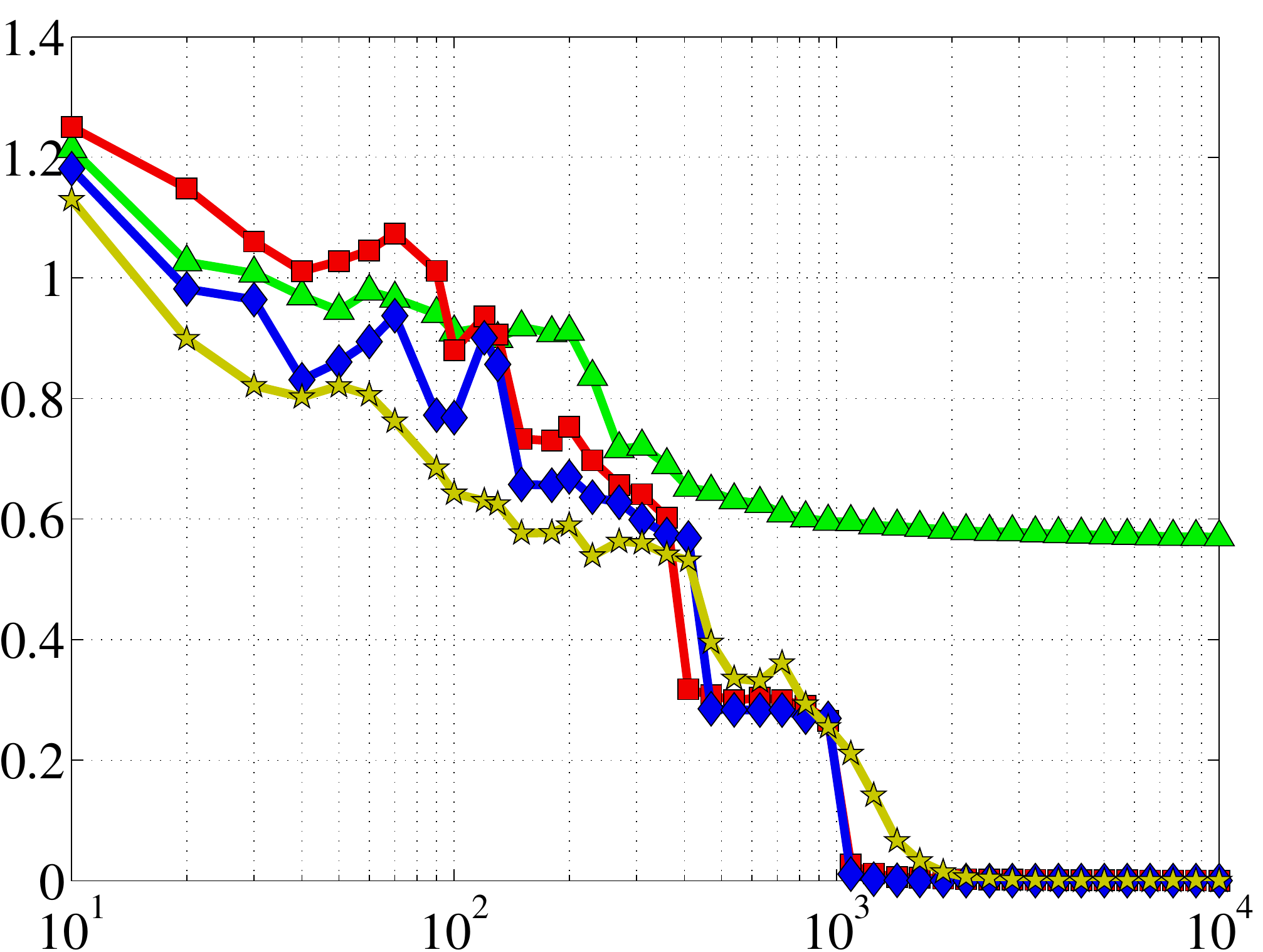} \\
& \scriptsize{Iterations} & \scriptsize{Iterations} & \scriptsize{Iterations}
\end{tabular}
\end{center}
\caption{
\small
Comparison on simulated data for different values of parameter $k$. 
}
\label{fig:simulated-iterations}
\end{figure*}

The incremental algorithm was found to perform extremely well in
practice--it was the best, in fact, among the compared algorithms \citep{Allerton}.
However, there exist cases in which the incremental algorithm can get
stuck at a suboptimal solution. For example, If the data are drawn
from a discrete distribution $\mathcal{D}$ which samples
$[\sqrt{3},0]^T$ with probability $1/3$ and $[0,\sqrt{2}]^T$ with
probability $2/3$, and one runs the incremental algorithm with $k=1$,
then it will converge to $[1,0]^T$ with probability $5/9$, despite the
fact that the maximal eigenvector is $[0,1]^T$.  The reason for this
failure is essentially that the orthogonality of the data interacts
poorly with the low-rank projection: any update which does not
entirely displace the maximal eigenvector in one iteration will be
removed entirely by the projection, causing the algorithm to fail to
make progress. Capped MSG algorithm with $K>k$, will not get stuck in
such situations, using the additional ``dimensions'' to ``search'' in
the new direction.  Only as it becomes more confident in its current
candidate, the trace of $M$ will become increasingly concentrated on
the top $k$ directions. To illustrate this empirically, we generalized the toy example above and generated the data using the $32$-dimensional ``orthogonal'' distribution described in Sec.~\ref{subsec:low-rank}. This distribution  presents challenging test-cases for MSG, capped MSG as well as incremental algorithm. Figure \ref{fig:simulated-iterations} shows plots of individual runs of MSG, capped MSG with $K=k+1$, the incremental algorithm, and Warmuth and Kuzmin's algorithm, all based on the same sequence of samples drawn from the orthogonal distribution. We compare algorithms in terms of the suboptimality on the population objective based on the largest $k$ eigenvalues of the state matrix $M^{(t)}$. 
The plots show the incremental algorithm getting 
stuck for $k\in\{1,4\}$, and the others intermittently plateauing at intermediate solutions before beginning to again converge rapidly towards the optimum. This behavior is to be expected on the capped MSG algorithm, due to the fact that the dimension of the subspace stored at each iterate is constrained. However, it is somewhat surprising that MSG and Warmuth and Kuzmin's algorithm behaved similarly, and barely faster than capped MSG.

\section{Experiments}
\label{sec:experiments}

\begin{figure*}[t]
\begin{center}
\begin{tabular}{ @{} L @{} S @{} S @{} S @{} }
& \large{$k=1$} & \large{$k=4$} & \large{$k=8$} \\
\rotatebox{90}{\scriptsize{Suboptimality}} &
\includegraphics[width=0.32\textwidth]{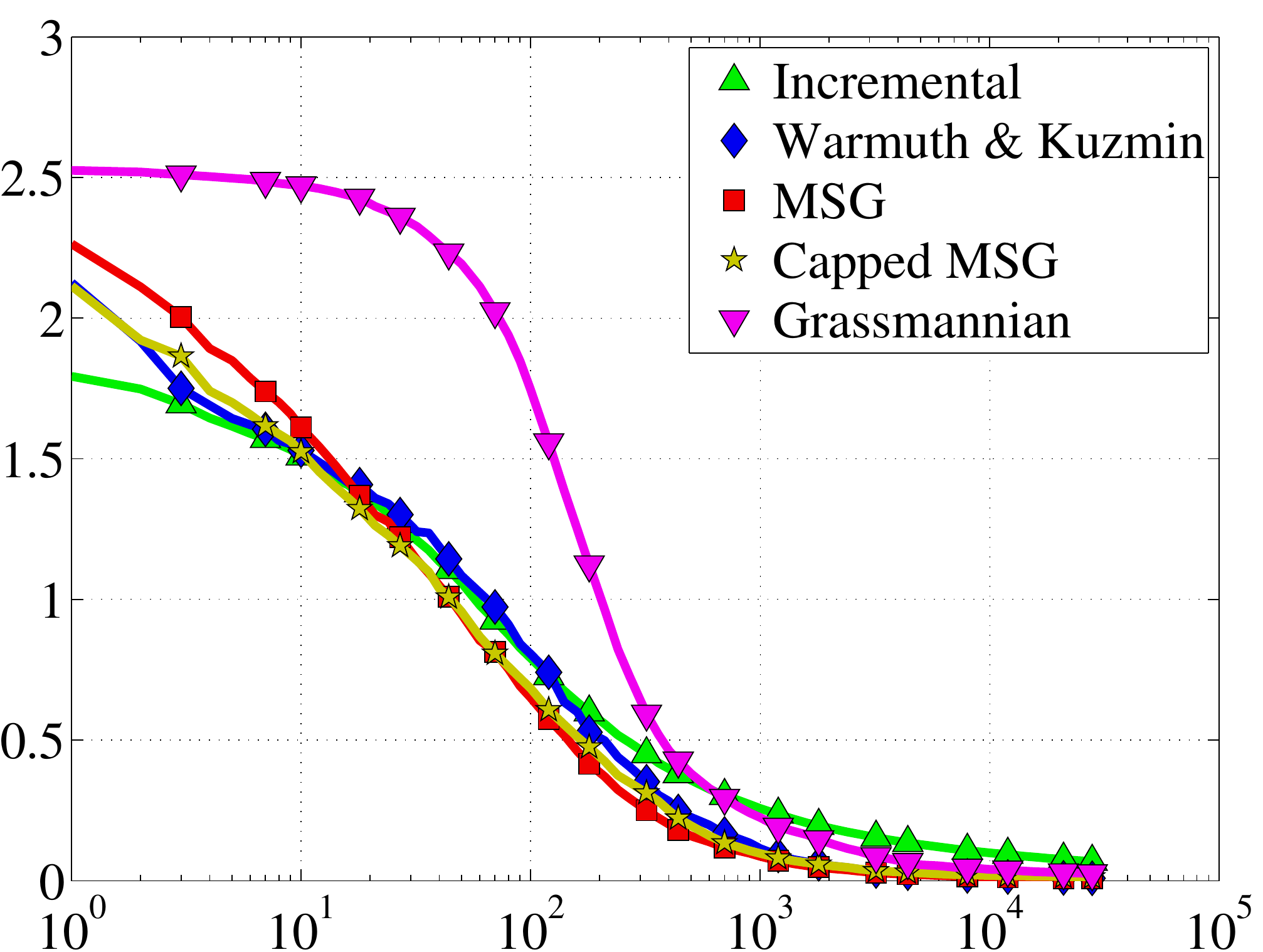} &
\includegraphics[width=0.32\textwidth]{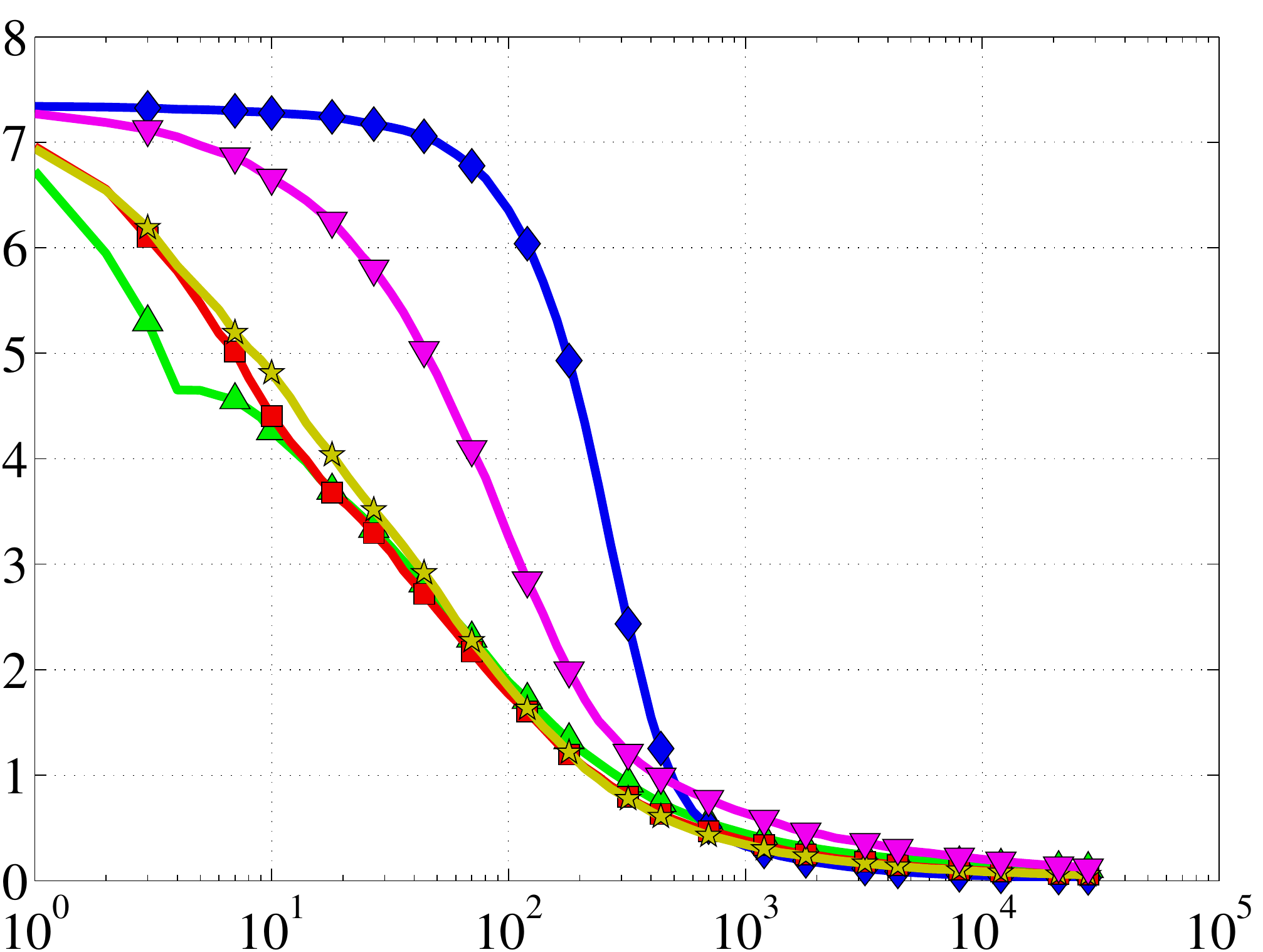} &
\includegraphics[width=0.32\textwidth]{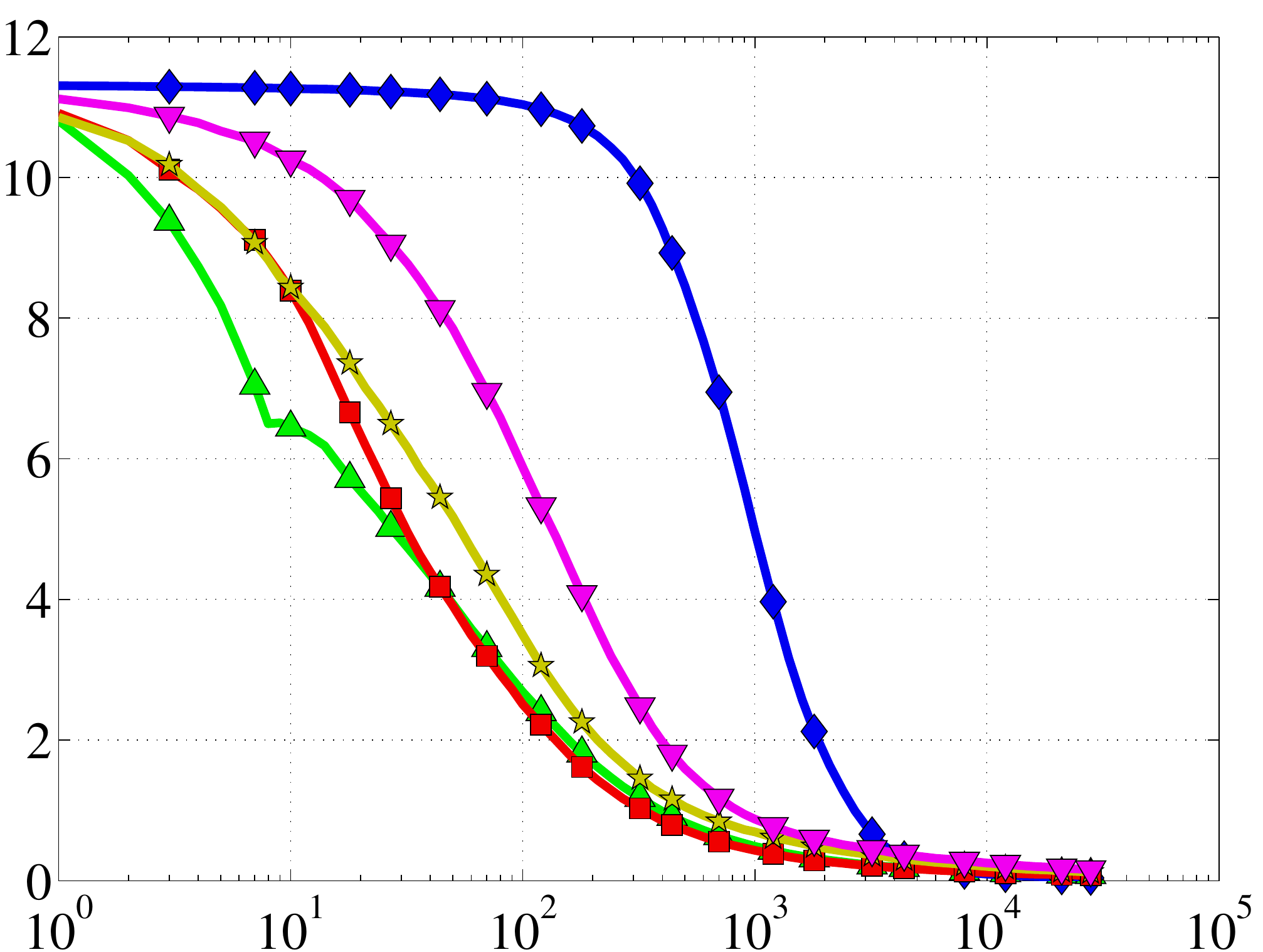} \\
& \scriptsize{Iterations} & \scriptsize{Iterations} & \scriptsize{Iterations}\\
\rotatebox{90}{\scriptsize{Suboptimality}} &
\includegraphics[width=0.32\textwidth]{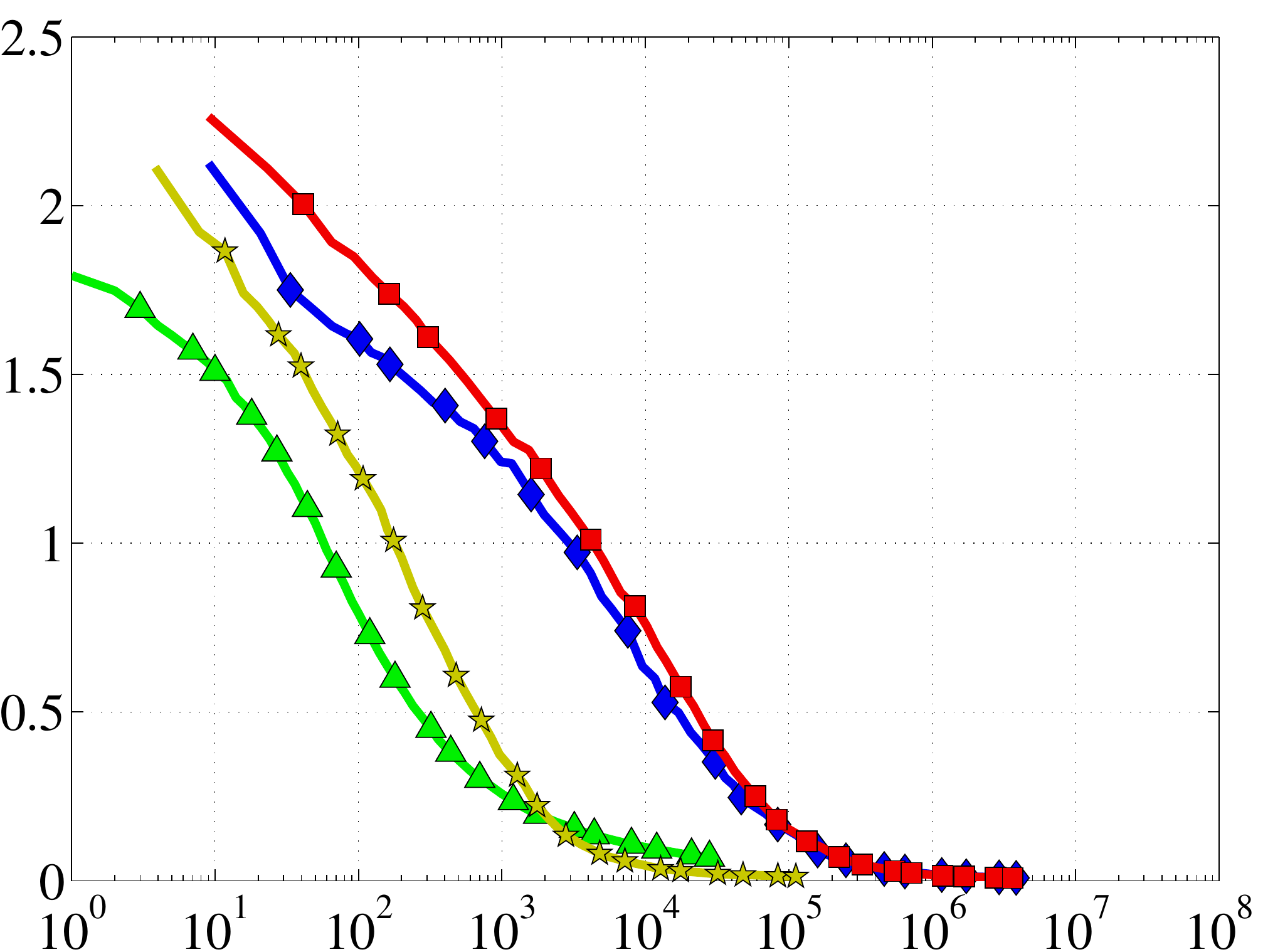} &
\includegraphics[width=0.32\textwidth]{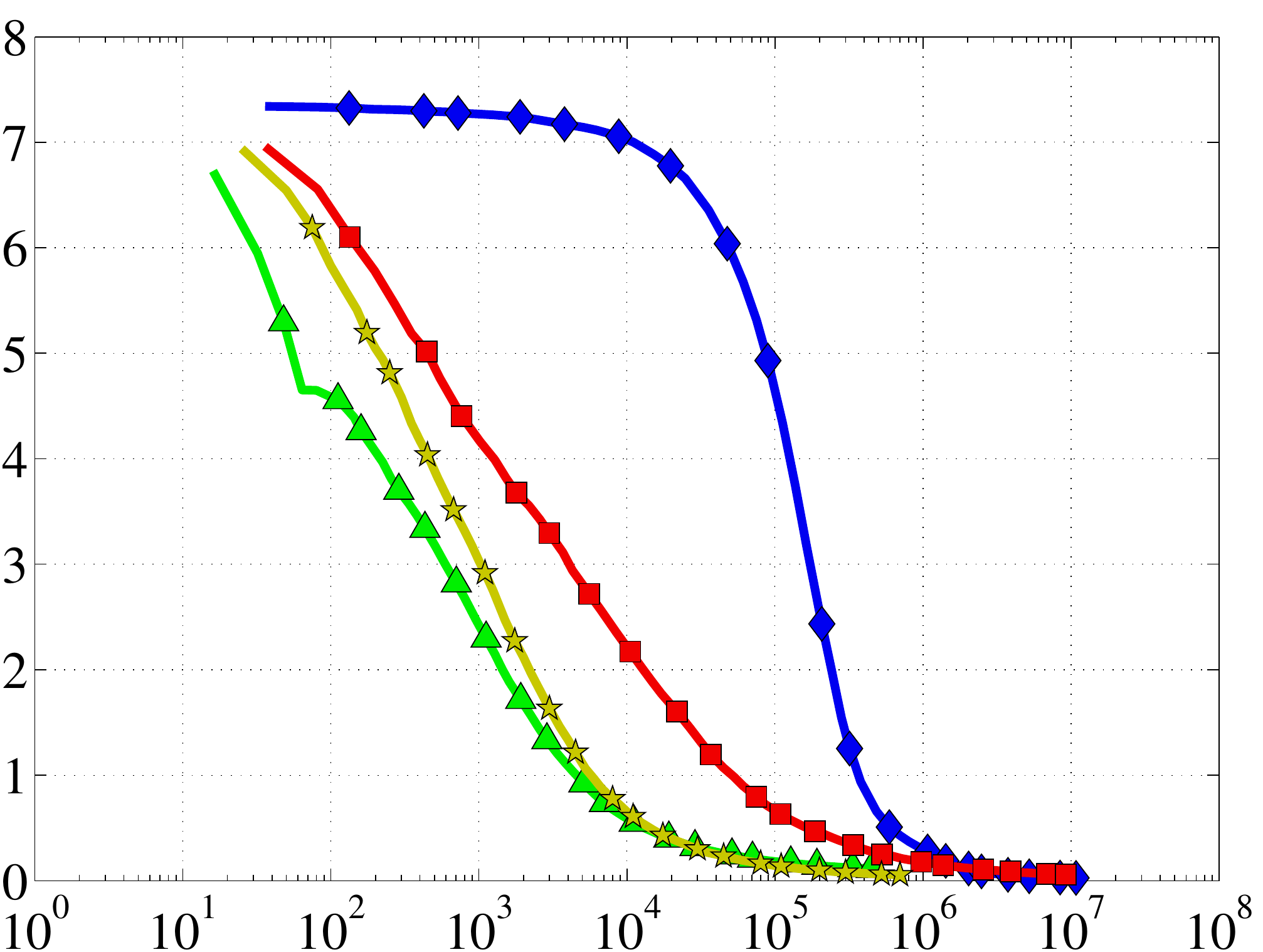} &
\includegraphics[width=0.32\textwidth]{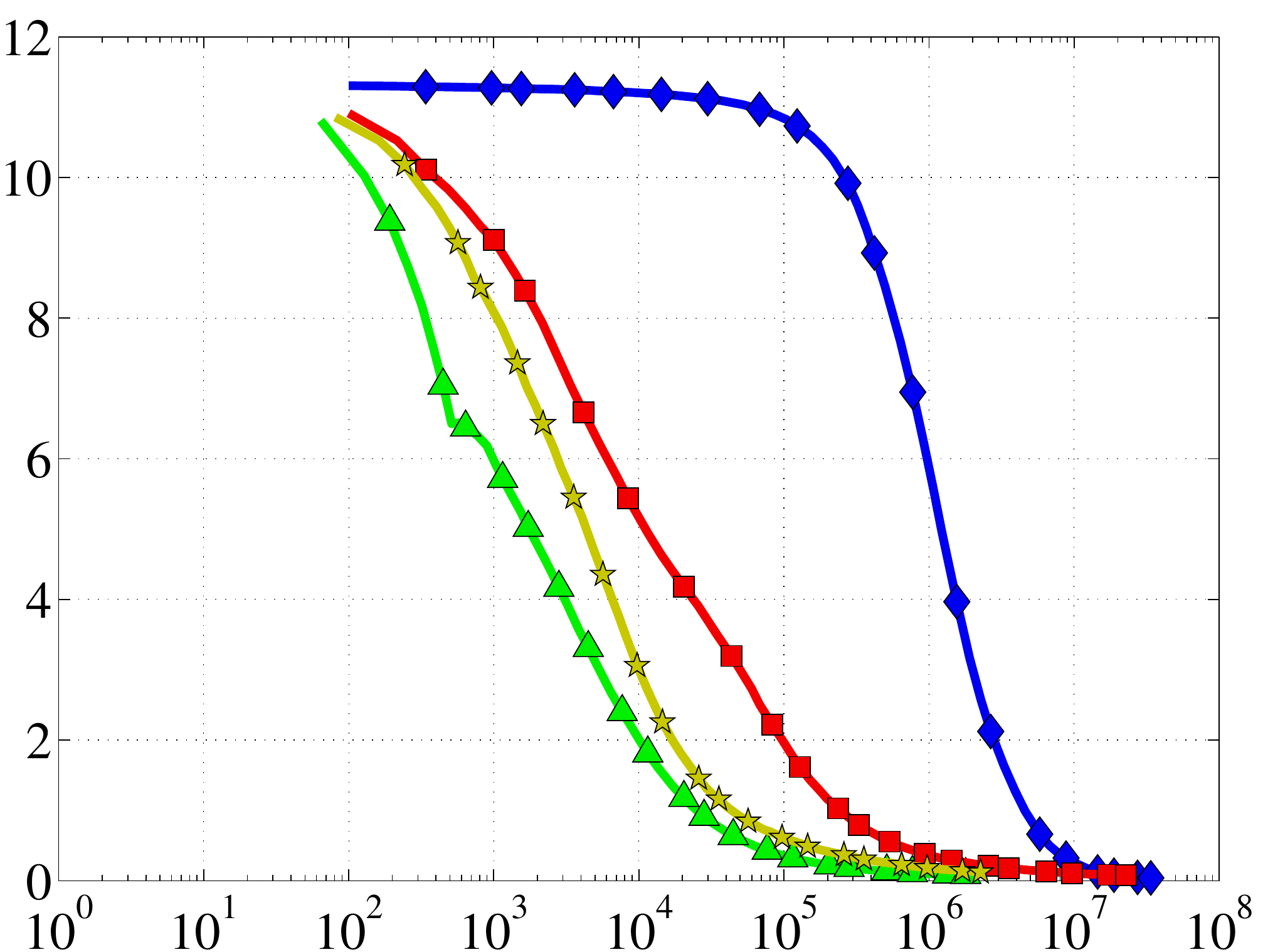} \\
& \scriptsize{Est. runtime} & \scriptsize{Est. runtime} & \scriptsize{Est. runtime}
\end{tabular}
\end{center}
\caption{
\small
Comparison on the MNIST dataset. The top row of plots shows suboptimality as a function of iteration count, while the bottom row suboptimality as a function of estimated runtime $\sum_{s=1}^{t} (k_s')^2$.
}
\label{fig:real}
\end{figure*}

We also compared the algorithms on the real-world MNIST dataset, which consists of $70,000$ binary images of handwritten digits of size $28 \times 28$, resulting in a dimensionality of $784$. We pre-normalized the data by mean centering the feature vectors and scaling each feature by the product of its standard deviation and the data dimension, so that each feature vector is zero mean and unit norm in expectation. In addition to MSG, capped MSG, the incremental algorithm and Warmuth and Kuzmin's algorithm, we also compare to a Grassmannian SGD algorithm of \citet{GROUSE:2010}. All algorithms except the incremental algorithm have a step-size parameter. In these experiments, we ran each algorithm with decreasing step sizes $\eta_t = c/\sqrt{t}$ for $c \in \{2^{-12}, 2^{-19}, \ldots, 2^5\}$ and picked the best $c$, in terms of the average suboptimality over the run, on a validation set. Since we cannot evaluate the true population objective, we estimate it by evaluating on a held-out test set. We use 40\% of samples in the dataset for training, 20\% for validation (tuning step-size), and 40\% for testing. We are interested in learning a maximum variance subspace of dimension $k\in\{1,4,8\}$ in a single ``pass'' over the training sample. In order to compare MSG, capped MSG, incremental and Warmuth and Kuzmin's algorithm in terms of runtime, we calculate the dominant term in the computational complexity: $\!\sum_{s=1}^{t}\!(k_s')^2\!$. The results are averaged over $100$ random splits into train-validation-test~sets. 

We can see from Figure \ref{fig:real} that the incremental algorithm makes the most progress per iteration and is also the fastest of all algorithms. MSG is comparable to the incremental algorithm in terms of the the progress made per iteration. However, its runtime is slightly worse than the incremental because it will often keep a slightly larger representation (of dimension $k_t'$) than the incremental algorithm. The capped MSG variant (with $K=k+1$) is significantly faster--almost as fast as the incremental algorithm, while, as we saw in the previous section, being less prone to getting stuck. Warmuth and Kuzmin's algorithm fares well with $k=1$, but its performance drops for higher $k$. Inspection of the underlying data shows that, in the $k\in\{4,8\}$ experiments, it also tends to have a larger $k_t'$ than MSG in these experiments, and therefore has a higher cost-per-iteration. Grassmannian SGD  performs better than Warmuth and Kuzmin, but much worse when compared with MSG and capped MSG.

\section{Conclusions}
\label{sec:discussion}
In this paper, we presented a careful development and analysis of MSG, a stochastic approximation algorithm for PCA, which enjoys good theoretical guarantees and offers a computationally efficient variant, capped MSG. We show that capped MSG is well-motivated theoretically and that it does not get stuck at a suboptimal solution. Capped MSG is also shown to have excellent empirical performance and it therefore is a much better alternative to the recently proposed incremental PCA algorithm of \cite{Allerton}. Furthermore, we provided a cleaner interpretation of PCA updates of \cite{Warmuth:08} in terms of Matrix Exponentiated Gradient (MEG) updates and 
showed that both MSG and MEG can be interpreted as mirror descent algorithms on the same relaxation of the PCA optimization problem but with different distance generating functions.

\bibliographystyle{plainnat}
\bibliography{main}

\newpage
\onecolumn
\appendix

\section{Proof of Lemma \ref{lem:sgd-projection}}\label{app:proofs}

\begin{replem}{lem:sgd-projection}
Let $M'\in\R^{d\times d}$ be a symmetric matrix, with eigenvalues 
$\sigma_{1}',\dots,\sigma_{d}'$ and associated eigenvectors $v_{1}',\dots,v_{d}'$. 
If $M=\project{M'}$ projects $M'$ onto the feasible region
of Problem \ref{eq:convex-objective} with respect to the Frobenius norm, then
$M$ will be the unique feasible matrix which has the same set of eigenvectors
as $M'$, with the associated eigenvalues $\sigma_{1},\dots,\sigma_{d}$
satisfying:
\begin{equation*}
\sigma_{i} = \max\left(0,\min\left(1,\sigma_{i}'+S\right)\right)
\end{equation*}
with $S\in\R$ being chosen in such a way that
$\displaystyle\sum_{i=1}^{d}\sigma_{i}=k$.
\end{replem}
\begin{proof}
The problem of finding $M$ can be written in the form of a convex optimization
problem as:
\begin{align*}
\mbox{minimize} : &\ \ \norm{M-M'}_F^2\\
\mbox{subject to} : &\ \ 0 \preceq M \preceq I, \trace M = k.
\end{align*}
Because the objective is strongly convex, and the constraints are convex, this
problem must have a unique solution. Letting $\sigma_{1},\dots,\sigma_{d}$ and
$v_{1},\dots,v_{d}$ be the eigenvalues and associated eigenvectors of $M$, we
may write the KKT first-order optimality conditions \citep{KKT} as:
\begin{equation}
\label{eq:kkt} 0 = M-M' + \mu I - \sum_{i=1}^{d} \alpha_{i}v_{i}v_i^T +
\sum_{i=1}^d \beta_{i}v_{i}v_{i}^{T}, 
\end{equation}
where $\mu$ is the Lagrange multiplier for the constraint $\trace M = k$, and
$\alpha_i, \beta_i \ge 0$ are the Lagrange multipliers for the constraints $0
\preceq M$ and $M \preceq I$, respectively.  The complementary slackness
conditions are that $\alpha_{i}\sigma_{i} = \beta_{i}\left(\sigma_{i}-1\right)
= 0$. In addition, $M$ must be feasible.

Because every term in Equation \ref{eq:kkt} \emph{except} for $M'$ has the same
set of eigenvectors as $M$, it follows that an optimal $M$ must have the same
set of eigenvectors as $M'$, so we may take $v_i = v_i'$, and write Equation
\ref{eq:kkt} purely in terms of the eigenvalues:
\begin{equation*}
\sigma_{i} = \sigma_{i}'-\mu+\alpha_{i}-\beta_{i}. 
\end{equation*}
Complementary slackness and feasibility with respect to the constraints
$0 \preceq M \preceq I$ gives that if $0 \le \sigma_{i}'-\mu\le 1$,
then $\sigma_{i}=\sigma_{i}'-\mu$. Otherwise, $\alpha_{i}$ and $\beta_{i}$
will be chosen so as to clip $\sigma_{i}$ to the active constraint:
\begin{equation*}
\sigma_i = \max\left( 0, \min\left( 1, \sigma_{i}'-\mu \right) \right). 
\end{equation*}
Primal feasibility with respect to the constraint $\trace M = k$ gives that
$\mu$ must be chosen in such a way that $\trace M = k$, completing the proof.
\end{proof}

\end{document}